\documentclass[svgnames]{article}

\usepackage{microtype}
\usepackage{graphicx}
\usepackage{subcaption}
\usepackage{booktabs} 
\usepackage{times}
\usepackage{epsfig,wrapfig}

\usepackage{url}            
\usepackage{amsfonts}       
\usepackage{nicefrac}       
\usepackage{amssymb}
\usepackage{amsmath}
\usepackage{amsthm}
\usepackage{subcaption}
\usepackage{tikz}
\usepackage{tkz-euclide}
\usepackage{subcaption}
\usetikzlibrary{calc}

\usepackage{algpseudocode}
\usepackage[linesnumbered,ruled]{algorithm2e}
\usepackage{xfrac}
\usepackage{color}
\usepackage{paralist}
\usepackage{mathtools}
\usepackage{cite}
\usepackage{comment}

\usepackage{multirow}
\usepackage{colortbl, array}
\usepackage{hhline}

\usepackage{hyperref}

\usepackage[accepted]{icml2020_modified}

\icmltitlerunning{Manifold-valued Convolutional Network with Applications}

\newcommand{\highest}[1]{\textcolor{Maroon}{\mathbf{#1}}}

\definecolor{rulecolor}{RGB}{0,71,171}
\definecolor{tableheadcolor}{RGB}{204,229,255}

\newcommand{\topline}{ %
        \arrayrulecolor{rulecolor}\specialrule{0.1em}{\abovetopsep}{0pt}%
        \arrayrulecolor{tableheadcolor}\specialrule{\belowrulesep}{0pt}{0pt}%
        \arrayrulecolor{rulecolor}}
\newcommand{\midtopline}{ %
        \arrayrulecolor{tableheadcolor}\specialrule{\aboverulesep}{0pt}{0pt}%
        \arrayrulecolor{rulecolor}\specialrule{\lightrulewidth}{0pt}{0pt}%
        \arrayrulecolor{white}\specialrule{\belowrulesep}{0pt}{0pt}%
        \arrayrulecolor{rulecolor}}
\newcommand{\bottomline}{ %
        \arrayrulecolor{white}\specialrule{\aboverulesep}{0pt}{0pt}%
        \arrayrulecolor{rulecolor} %
        \specialrule{\heavyrulewidth}{0pt}{\belowbottomsep}}%

\newcommand{\argmin}{\operatornamewithlimits{argmin}}

\newtheorem{theorem}{Theorem}
\newtheorem{lemma}{Lemma}
\newtheorem{definition}{Definition}
\newtheorem*{remark*}{Remark}

\makeatletter
\newsavebox{\mybox}\newsavebox{\mysim}
\newcommand{\distras}[1]{%
  \savebox{\mybox}{\hbox{\kern3pt$\scriptstyle#1$\kern3pt}}%
  \savebox{\mysim}{\hbox{$\sim$}}%
  \mathbin{\overset{#1}{\kern\z@\resizebox{\wd\mybox}{\ht\mysim}{$\sim$}}}%
}

\begin{document}

\twocolumn[
\icmltitle{MVC-Net: A Convolutional Neural Network Architecture for Manifold-Valued Images With Applications}
\begin{icmlauthorlist}
\icmlauthor{Jose Bouza\textsuperscript{*}}{cise}
\icmlauthor{Chun-Hao Yang\textsuperscript{*}}{stat}
\icmlauthor{David Vaillancourt}{phys}
\icmlauthor{Baba C. Vemuri}{cise}
\end{icmlauthorlist}

\icmlaffiliation{cise}{Department of CISE, University of Florida, Florida, USA}
\icmlaffiliation{phys}{Department of Applied Physiology and Kinesiology, University of Florida, USA}
\icmlaffiliation{stat}{Department of Statistics, University of Florida, Florida, USA}

\icmlcorrespondingauthor{Baba C. Vemuri}{baba.vemuri@gmail.com}

\icmlkeywords{Machine Learning, ICML}
\vskip 0.3in
]

\printAffiliationsAndNotice{\textsuperscript{*} \textbf{Equal Contribution} \newline}

\begin{abstract}
Geometric deep learning has attracted significant
attention in recent years, in part due to the availability
of exotic data types for which traditional
neural network architectures are not well suited.
Our goal in this paper is to generalize 
convolutional neural networks (CNN)
to the manifold-valued image case which
arises commonly in medical imaging and computer
vision applications. Explicitly, the input
data to the network is an image where each pixel value
is a sample from a Riemannian manifold.
To achieve this goal, we must generalize the basic
building block of traditional CNN
architectures, namely, the weighted combinations
operation. To this end, we develop a tangent space
combination operation which is used to define a
convolution operation on manifold-valued images
that we call, the Manifold-Valued Convolution
(MVC). We prove theoretical properties of the
MVC operation, including equivariance to the
action of the isometry group admitted by the manifold
and characterizing when compositions of
MVC layers collapse to a single layer. We present
a detailed description of how to use MVC layers
to build full, multi-layer neural networks that operate
on manifold-valued images, which we call
the MVC-net. Further, we empirically demonstrate
superior performance of the MVC-nets in
medical imaging and computer vision tasks.
\end{abstract}

\section{Introduction}\label{sec:introduction}

In computer vision, convolutional neural networks (CNN) and its variants are ubiquitous and serve as omnipotent tools for various tasks, e.g.\ image classification and segmentation. However, the traditional CNNs are restricted to data residing in vector spaces while data residing in smooth non-Euclidean spaces, e.g.\ Riemannian manifolds, arise naturally in many problem domains. Although Riemannian manifolds lack the vector space structure, the associated Riemannian metric induces the notions of distance and angle (between intersecting curves on the manifold) intrinsic to the manifold. Commonly encountered examples of Riemannian manifolds in computer vision are the manifold of $(n \times n)$ symmetric positive-definite (SPD) matrices, $P_n$, the special orthogonal group $\textsf{SO}(n)$, the Grassmann manifold, $\textsf{Gr}(n,p)$ and the $n$-sphere, $S^n$. Recently, there has been a growing interest in generalizing the well-known CNN and its variants to cope with these types of data while respecting the underlying geometry.

In the past few years, there has been a surge in research to develop deep neural networks that deal with data residing on the aforementioned Riemannian manifolds. At the outset, it will be useful to categorize two types of problems concerning data in non-Euclidean spaces. These two types are: (i) data that are samples of functions defined on smooth manifolds and (ii) data that are samples of manifold-valued functions whose domain is Euclidean or data that are simply sample points on manifolds. In this paper we will address the problem of developing deep neural networks for the data defined in (ii). 

In the context of data defined in (i) above, in the recent past, there has been a flurry of research activity in developing the analogs of CNNs.  For example, \citet{masci2015geodesic} presented the geodesic convolutional neural network (GCNN) for which they defined the geodesic convolution as standard convolution on the local geodesic charts. \citet{poulenard2018multi} presented convolution for directional functions which reduces to the usual convolution when the underlying manifold is $\mathbb{R}^d$. 
In both \citet{masci2015geodesic} and \citet{poulenard2018multi}, convolutions are performed in local geodesic polar charts constructed on the manifold. Moving on, samples of functions defined on a sphere are encountered in numerous applications of computer vision and to this end, there is the spherical CNN work reported in \citet{esteves2018learning}, \citet{kondor2018clebsch}, and \citet{cohen2018spherical}. In this problem, group equivariant convolutions were used to replace the standard convolutions in CNNs. Note that the group action on the sphere corresponds to  rotations in 3D which are members of the group $SO(3)$. Recently, the equivariance of convolutions to more general classes of group actions suited for other Riemannian homogeneous spaces has been reported in \citet{kondor2018generalization}, \citet{banerjee2019dmr}, and \citet{cohen2018spherical}. We will not discuss methods suited for this type of data any further in this paper but refer the reader to \citet{bronstein2017geometric} who present a good survey of state-of-the-art in geometric deep learning.
  
In the context of data described in (ii) above,  \citet{huang2017riemannian} proposed a network architecture that consisted of layers which explicitly utilize the structure of SPD matrices. \citet{huang2018building} presented a deep network for classification of hand-crafted features residing on a Grassmann manifold. However, the above architectures do not resemble the classic convolutional layer in the traditional CNN which is viewed as one of the key component to the success of CNNs. Furthermore, the operations used in the above network are not valid for general Riemannian manifolds. For example, in \citet{huang2017riemannian}, applying ReLU and logarithms on the eigenvalues is not valid for Grassmann manifolds. Besides convolutional layers, batch normalization is also a useful trick in CNN to avoid over-fitting and \citet{brooks2019riemannian} proposed a batch normalization technique for manifold-valued networks. In this paper, we focus our attention to data represented on a grid where each of the grid points is associated with a value on a known manifold, e.g.\ $f: Z^2 \to M$. However, all the aforementioned works are targeted for specific manifolds, e.g.\ the  Grassmann or the SPD manifolds. The lack of a consistent framework for designing deep network architectures for data residing on a general Riemannian manifold is partly due to the fact that there is no natural analog of convolution operation for manifold-valued data. This justifies the need to generalize the convolution operation for data in Riemannian manifolds in order to develop a consistent framework for deep learning to tackle such data. Recently, \citet{chakraborty2019deep} proposed to use weighted Fr\'{e}chet mean (\textsf{wFM}) \cite{Frechet1948a} as an analog to the classical (Euclidean space) convolution operation for data points residing in Riemannian manifolds. Note that although their definition of \textsf{wFM} as an analogous operation is valid for any Riemannian manifold, the convexity constraints in the definition used for \textsf{wFM} puts certain restrictions on the range of values that the \textsf{wFM} can take on and this can limit the modeling capacity of the network as we will see later.  

In order to generalize the (discrete) convolution operation in Euclidean spaces -- which is simply a linear combination of weights and image function values inside a certain window -- to Riemannian manifolds, we have to define what is a meaningful ``equivalent'' of the aforementioned linear combination operation in the Riemannian manifold setting.  In this paper, we propose to make use of the idea that it is possible to map the manifold-valued data points within a convolution window defined over the manifold-valued image to the tangent space anchored at the FM of these points using the Riemannian $\textbf{Log}$ map. Then, perform the linear combination operation in the tangent space (which is isomorphic to the Euclidean space) and map it back to the manifold using the Riemannian $\textbf{Exp}$ map. We provide the details of this operation called the manifold-valued convolution (MVC) in the next section. Further, we prove that the proposed MVC is equivariant to the isometry group actions admitted by the manifold. 
Armed with MVC, we then describe how to build a MVC-Net for manifold-valued data by defining the corresponding activation functions and fully-connected (FC) layers for the manifold-valued data. {\it Thus, the main contributions of our work in this paper are the following. (i) We define the MVC operation for general Riemannian manifolds and a prove that MVC is equivariant to isometry group actions admitted by the manifold. (ii) We present a deep neural network architecture based on MVC, called MVC-Net, for \textbf{any} Riemannian manifold. (iii) Further, we present experiments demonstrating performance of the MVC-Net on classification problems encountered in medical image analysis and computer vision along with comparisons to the state-of-the-art.}
 
The rest of this paper is organized as follows. In section~\ref{sec:background}, we review some essential background in Riemannian geometry. In section~\ref{sec:architecture}, we propose a novel generalization, the MVC, of the convolution operation for Riemannian manifold-valued images and show that MVC is equivariant to isometry group actions admitted by the manifold. Then, we propose a deep neural network architecture based on MVC, called the MVC-Net. In section~\ref{sec:experiment}, we present the experimental results and finally draw conclusions in section~\ref{sec:conclusion}.  
\section{Review of Riemannian Geometry} \label{sec:background}
In this section, we review some basic material from Riemannian geometry that is necessary in our work.

Let $(M,g)$ be a $d$-dimensional Riemannian manifold. For $p \in M$, the \emph{tangent space} of $M$ at $p$ is denoted $T_pM$, which is a $d$-dimensional vector space. Equipped with the Levi-Civita connection, the geodesic starting at $p$ is denoted $\gamma_v:I\to M$ with $\gamma_v(0)=p$ where $I$ is some interval containing $0$, and $v$ is the initial tangent vector, i.e.\ $\gamma^{\prime}_v(0) = v$. Sometimes a geodesic is specified by the two endpoints $p,q$ and in this case we denote the geodesic by $\gamma_{p,q}$ such that $\gamma_{p,q}(0)$ and $\gamma_{p,q}(1)=q$. The \emph{Exponential map} $\textbf{Exp}_p:D(p) \subset T_pM \to M$ is defined by $\textbf{Exp}_p(v)=\gamma_v(1)$ where $D(p)=\{v\in T_pM: \gamma_v(1)\text{ is defined}\}$. The exponential map is a diffeomorphism from $D(p)$ to its range, and its inverse is denoted $\textbf{Log}_p= \textbf{Exp}_p^{-1}$. These two maps will be of fundamental importance for our proposed layer which is discussed later in this section. 

In general, there is no global coordinate system on a Riemannian manifold. Therefore, a local coordinate system is important for doing computations on Riemannian manifolds. The most common one is called the \emph{normal coordinate} which is based on the Riemannian exponential map and the log map respectively. The normal coordinates are constructed as follows. For $p \in M$, there exist a neighborhood $U_p \subset M$ of $p$ and a neighborhood $V \subset T_pM$ such that $\textbf{Exp}_p$ is a diffeomorphism between $V$ and $U_p$ (Lemma 5.10 in \citet{lee2006riemannian}). The neighborhood $U_p$ is called the \emph{normal neighborhood}. The \emph{normal coordinate} of $q \in U_p$ with respect to the normal neighborhood $U_p$ is given by $\textbf{Log}_p(q)$. This concept is important as we will use it in the definition of manifold-valued convolution in section~\ref{sec:architecture}.

The Riemannian metric $g$, induces a distance given by,
\[
    d_g(p,q)=\inf\Bigg\{\int_0^1 \sqrt{g(\gamma^{\prime}_{p,q}(t), \gamma^{\prime}_{p,q}(t))}dt: \text{for all } \gamma_{p,q} \Bigg\}.
\]
Let $x_1,\ldots,x_n \in M$. The Fr\'{e}chet mean (FM) of $x_1,\ldots,x_n$ is 
\[
    \bar{x} = \argmin_{m\in M}\sum_{i=1}^n d^2_g(x_i,m)
\]
This is a generalization of mean of points in a vector space. The existence and uniqueness of the FM is discussed in \citet{Afsari2011}. To be precise, the FM is unique if $x_1,\ldots,x_n$ lie in a open ball of radius $r_{\text{cvx}}$, where $r_{\text{cvx}}$ is the \emph{convexity radius} of $M$ \cite{Groisser2004}. In practice, it is always assumed to be this case.

With this intrinsic distance metric, the Riemannian manifold $(M, d_g)$ is a metric space and a natural transformation under consideration would be the \emph{isometry}. For a Riemannian manifold, a transformation $\phi:(M,g) \to (\tilde{M}, \tilde{g})$ is called an isometry if it is a diffeomorphism and $g = \phi^*\tilde{g}$ where $\phi^*$ is pullback operation of $\phi$. In this work, we consider the isometry from $M$ to $M$. It is known that the collection of isometries forms a group under composition, denoted $I(M)$. For a smooth map $f:M\to M$, a desired property would be the \emph{isometry equivariant}, i.e.\ $\phi \circ f = f \circ \phi$. Another similar concept is the \emph{isometry invariance}, i.e.\ $f \circ \phi = f$.

\begin{remark*}
Note that with a slight abuse of notation, for a metric space $X$, we denote the set of all isometry transformations of $X$ by $I(X)$ as well.
\end{remark*}

\section{MVC-Net Theory and Architecture}

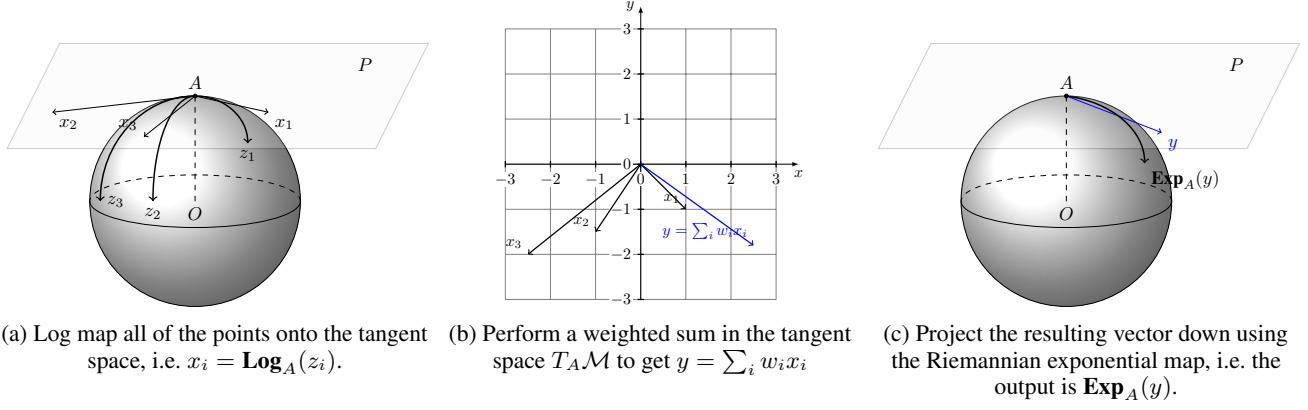
\begin{figure*}[htb]
\captionsetup[subfigure]{justification=centering}
\centering
\begin{subfigure}[t]{.33\textwidth}
  \centering
  \scalebox{0.7}{
    \begin{tikzpicture}[
      point/.style = {draw, circle, fill=black, inner sep=0.7pt},
    ]
        \def\rad{2cm}
        \coordinate (O) at (0,0); 
        \coordinate (N) at (0,\rad); 
        
        \filldraw[ball color=white] (O) circle [radius=\rad];
        \draw[dashed] 
          (\rad,0) arc [start angle=0,end angle=180,x radius=\rad,y radius=5mm];
        \draw
          (\rad,0) arc [start angle=0,end angle=-180,x radius=\rad,y radius=5mm];
        \begin{scope}[xslant=0.5,yshift=\rad,xshift=-2]
        \filldraw[fill=gray!10,opacity=0.2]
          (-4,1) -- (3,1) -- (3,-1) -- (-4,-1) -- cycle;
        \node at (2,0.6) {$P$};  
        \end{scope}
        \draw[dashed]
          (N) node[above] {$A$} -- (O) node[below] {$O$};
        \node[point] at (N) {};
        \draw[thick, ->] (0,2,0) arc [start angle=-90,end angle=0,x radius=1,y radius=-0.9] node[below] {$z_1$}; 
        \draw[thick, ->] (0,2,0) arc [start angle=-90,end angle=0,x radius=-0.8,y radius=-2] node [below] {$z_2$};
        \draw[thick, ->] (0,2,0) arc [start angle=-90,end angle=0,x radius=-1.8,y radius=-2] node [right] {$z_3$};
        \draw[->] (0,2,0) -- (1.7,2,0.8) node[anchor=north west]{$x_1$};
        \draw[->] (0,2,0) -- (-2.4,2,0.8) node[anchor=north west]{$x_2$};
        \draw[->] (0,2,0) -- (-0.2,2,2) node[anchor=south east]{$x_3$};
    \end{tikzpicture}}
    \subcaption{Log map all of the points onto the tangent space, i.e.\ $x_i = \textbf{Log}_A(z_i)$.}
  \label{fig:sub1}
\end{subfigure}~
\begin{subfigure}[t]{.33\textwidth}
  \centering
  \scalebox{0.6}{
      \begin{tikzpicture}
      \tkzInit[xmax=3,ymax=3,xmin=-3,ymin=-3]
   \tkzGrid
   \tkzAxeXY
   \draw[thick,->] (0,0) -- (1,-1) node[anchor=south east]{$x_1$};
   \draw[thick,->] (0,0) -- (-1,-1.5) node[anchor=south east]{$x_2$};
   \draw[thick,->] (0,0) -- (-2.5,-2) node[anchor=south east]{$x_3$};
   \draw[thick,->,blue] (0,0) -- (2.5,-1.8) node[anchor=south east]{$y = \sum_i w_i x_i$};
    \end{tikzpicture}}
  \label{fig:sub2}
  \subcaption{Perform a weighted sum in the tangent space $T_A \mathcal{M}$ to get $y = \sum_i w_i x_i$}
\end{subfigure}~
\begin{subfigure}[t]{.33\textwidth}
  \centering
    \scalebox{0.7}{
      \begin{tikzpicture}[
      point/.style = {draw, circle, fill=black, inner sep=0.7pt},
    ]
            \def\rad{2cm}
            \coordinate (O) at (0,0); 
            \coordinate (N) at (0,\rad); 
            
            \filldraw[ball color=white] (O) circle [radius=\rad];
            \draw[dashed] 
              (\rad,0) arc [start angle=0,end angle=180,x radius=\rad,y radius=5mm];
            \draw
              (\rad,0) arc [start angle=0,end angle=-180,x radius=\rad,y radius=5mm];
            \begin{scope}[xslant=0.5,yshift=\rad,xshift=-2]
            \filldraw[fill=gray!10,opacity=0.2]
              (-4,1) -- (3,1) -- (3,-1) -- (-4,-1) -- cycle;
            \node at (2,0.6) {$P$};  
            \end{scope}
            \draw[dashed]
              (N) node[above] {$A$} -- (O) node[below] {$O$};
            \node[point] at (N) {};
            \draw[->,blue] (0,2,0) -- (2.5,2,1.8) node[anchor=north west]{$y$};
            \draw[thick, ->] (0,2,0) arc [start angle=-80,end angle=0,x radius=1.8,y radius=-1.3] node[below right ]{$\textbf{Exp}_A(y)$};
            \end{tikzpicture}}
  \label{fig:sub2}
  \vspace{-12pt}
  \subcaption{Project the resulting vector down using the Riemannian exponential map, i.e.\
  the output is $\textbf{Exp}_A(y)$.}
\end{subfigure}
\caption{Tangent combination operation.}
\label{MVCFig}
\end{figure*}

\label{sec:architecture}
In this section, we present the MVC and show that it is equivariant under isometry group actions admitted by the manifold. Then we present the architecture of MVC-Net by introducing the basic constituent layers of the MVC-Net. 
\subsection{Manifold-valued convolution (MVC)}

Recall that in a CNN the convolution operation involves a linear combination of the data in the window, i.e.\ $\sum_{i=1}^n w_i x_i$. Due to the lack of vector space structure on Riemannian manifolds, we can not perform this usual convolution on manifold-valued images directly. In this work, we propose a generalization of the above described standard convolution to manifold-valed images, called the manifold-valued convolution (MVC) defined as follows. 
\begin{definition}
Let $(M,g)$ be a Riemannian manifold and $f:Z^n \to M$ and $w:Z^n \to \mathbb{R}$ be two functions defined on $Z^n$ where $Z$ is the set of all integers. The convolution, $f * w : Z^n \to M$ is defined by
\begin{align} \label{eqn:MVC}
    (f * w)(y) & = \textbf{Exp}_{m}\Bigg( \sum_{z \in Z^n} w(z-y) \textbf{Log}_{m}f(z)\Bigg) \nonumber \\
               & \coloneqq \text{MVC}(f, w)(y) 
\end{align}
for $y \in Z^n$ where $m = \textsf{FM}_{z \in Z^n}(f(z))$.
\end{definition}

An illustration of the MVC operation can be seen in Figure~\ref{MVCFig}
An important property of the convolution operation in Euclidean spaces is that it is equivariant to translation which is the natural isometry group action for Euclidean spaces. Thus, MVC, as a generalization of the convolution to Riemannian manifold-valued images, is expected to possess such property, i.e.\ equivariant to isometry group actions admitted by the manifold. The following lemma is useful for proving this result.

\begin{lemma}
Let $\phi:M \to M$ be an isometry. Then for $p \in M$ 
\[
\phi \circ \textbf{Exp}_p = \textbf{Exp}_{\phi(p)} \circ d\phi_p
\]
where $d\phi_p$ is the differential of $\phi$ at $p$. Therefore when the inverse of $\textbf{Exp}_p$ exists,
\[
\textbf{Log}_{\phi(p)} = d\phi_p \circ \textbf{Log}_p \circ \phi^{-1}.
\]
\end{lemma}
The proof of this lemma can be found in most of the introductory textbooks in Riemannian geometry, e.g.\ proposition 5.9 in \citet{lee2006riemannian}. 

\begin{theorem}\label{thm:equivariance}
The MVC is equivariant to isometry group actions both in the domains and the ranges of $f$ and $w$ i.e.\
for $f:Z^n \to M$ and $w : Z^n \to \mathbb{R}$,
\begin{align}
\text{MVC}( \phi \circ f, w ) & = \phi \circ \text{MVC}( f, w ), \quad \phi \in I(M)\label{eqn:dom_equi1}\\
\text{MVC}( f \circ \phi, w ) & = \text{MVC}( f, w ) \circ \phi, \quad \phi \in I(Z^n)\\
\text{MVC}( f, \phi \circ w ) & = \phi \circ \text{MVC}( f, w ), \quad \phi \in I(\mathbb{R})\\
\text{MVC}( f, w \circ \phi) & = \text{MVC}( f, w ) \circ \phi, \quad \phi \in I(Z^n).
\end{align}
\end{theorem}
\begin{proof}
We show only \eqref{eqn:dom_equi1} here since the other three equalities follow from similar derivation. First, note that the FM of $\phi \circ f$ is $\phi(\textsf{FM}_{z\in Z^n}(f(x))) \coloneqq \phi(m)$ where $m = \textsf{FM}_{z\in Z^n}(f(x))$.  This is a consequence of the invariance of the intrinsic distance metric. Then for $y \in Z^n$
\begin{align*}
& \quad \text{MVC}( \phi \circ f, w )(y)\\
& = \textbf{Exp}_{\phi(m)}\Big(\sum_{z \in Z^n} w(z-y) \textbf{Log}_{\phi(m)}(\phi\circ f)(z)\Big)\\
& = \textbf{Exp}_{\phi(m)}\Big( \sum_{z \in Z^n} w(z-y) \big(d\phi_m \circ \textbf{Log}_{m}\big)f(z)\Big)\\
& = \big(\textbf{Exp}_{\phi(m)} \circ d\phi_m \big)\Big(\sum_{z \in Z^n} w(z-y) \textbf{Log}_{m} f(z)\Big)\\
& = \big(\phi \circ \textbf{Exp}_m\big)\Big(\sum_{z \in Z^n} w(z-y) \textbf{Log}_{m} f(z)\Big)\\
& = (\phi \circ \text{MVC})(f, w )(y)
\end{align*}
This concludes the proof.
\end{proof}
Note that the equivariance is preserved even if the FM $m$ is replaced by any other points as long as the choice of the point is also equivariant, e.g.\ replace $m$ by $f(z_0)$, for some $z_0 \in Z^n$. This avoids the computation of the FM and hence is computationally more efficient. In practice, the analytic forms of $f$ and $w$ are unknown and only $x_i = f(z_i)$ and $w_i=w(z_i)$, $i=1,\ldots,n$ are observed for some fixed $z_1,\ldots z_N \in Z^n$. Thus from now on, we consider $\{x_i\}_{i=1}^N$ and $\{w_i\}_{i=1}^N$ instead of $f$ and $w$. For this situation, the MVC can be simplified as
\begin{align*}
    MVC(f, w) & = MVC\big( \{x_i\}_{i=1}^N, \{w_i\}_{i=1}^N\big)\\
              & = \textbf{Exp}_{\bar{x}}\Big(\sum_{i=1}^Nw_i\textbf{Log}_{\bar{x}}x_i\Big) 
\end{align*}
where $\bar{x} = \textsf{FM}\big(\{x_i\}_{i=1}^N\big)$. For applications in computer vision and medical imaging, the domain $Z^n$ is usually $Z^2$ or $Z^3$.


\subsection{Activation Functions for MVC-Net}

In classical neural networks, the activation functions, e.g.\ ReLU, sigmoid, tanh, etc., play an important role as they make the resulting network non-linear and thus we are able to build a deep neural network by stacking layers of different sizes along with the activation functions. The choice of activation functions has been studied extensively and there are a few guidelines for choosing one. First, the activation function must be a contraction map \cite{Mallat2016}. The precise definition of a contraction map will be given later. Second, the activation function should  prevent multiple stacked layers of the network from collapsing to a single layer, which allows us to build a deep network. In this section,  we will analyze the MVC-net in the context of the above guidelines. We first show that the MVC layer is not a contraction and under some conditions cascaded MVC layers will collapse into one. Then we give a possible choice of an activation function for use in the MVC-net.
\subsubsection*{\bf Contraction Property}
The following definition of contraction is from \citet{Mallat2016}.
\vspace*{-4pt}
\begin{definition}
Let $F: U \to V$ where $U$ and $V$ are metric spaces with distance metrics $d_U$ and $d_V$, respectively. The mapping $F$ is called a \emph{contraction} if for $x, y \in U$, there exists $c < 1$ such that $d_V(F(x), F(y)) < cd_U(x, y)$. If for all $x, y \in U$, $d_V(F(x), F(y)) < d_U(x, y)$, then $F$ is called a \emph{non-expansion}.
\end{definition}
\vspace*{-5pt}
Since the range of MVC is a normal neighborhood of the anchor point, it can be easily shown that the MVC layer is \emph{not} a contraction by considering large $w_i$'s. 

\subsubsection*{\bf Collapsibility Property}

In classical neural networks, one reason for adding  non-linear activation functions between layers, e.g.\ sigmoid, ReLU, tanh, is that without these, the multi-layer network collapses into a single-layer network. We want to know if a similar behavior is exhibited by the MVC-net. For example, consider a network with two MVC layers (without non-linear activation in between). For the sake of simplicity, suppose that there are only two MVC ``filters'' in the first MVC layer and one MVC ``filter'' in the second MVC layer, i.e.\ the first MVC layer takes $\{x_i\}_{i=1}^{2N}$ as input with weight $\{w_i\}_{i=1}^{2N}$ and the second MVC layer takes $\{M_1, M_2\}$ as input with weights $\{h_1,h_2\}$ where $M_1 = \text{MVC}\big(\{x_i\}_{i=1}^N, \{w_i\}_{i=1}^N\big)$ and $M_2=\text{MVC}\big(\{x_i\}_{i=N+1}^{2N},\{w_i\}_{i=N+1}^{2N}\big)$. Is this two-layer MVC-net equivalent to a one-layer MVC-net i.e., does there exist $\{\tilde{w}_i\}_{i=1}^{2N}$ such that $\text{MVC}(\{M_1,M_2\},\{h_1,h_2\}) = \text{MVC}\big(\{x_i\}_{i=1}^{2N},\{\tilde{w}_i\}_{i=1}^{2N}\big)$? We answer this question in the affirmative under some conditions as stated in the following theorem.

\begin{theorem}\label{thm:collapsibility}
Let $\{x_i\}_{i=1}^{2N} \subset M$. If $\{x_i\}_{i=1}^N$ and $\{x_i\}_{i=N+1}^{2N}$ belong to the same normal coordinate chart $U$ then, two cascaded MVC layers will collapse to a single layer.
\end{theorem}
\begin{proof}
As mentioned earlier, the anchor point of map \eqref{eqn:MVC} can be any point in the normal coordinate chart. Let $p \in U \subset M$ be such a point. Consider the weights $\{w_i\}_{i=1}^{2N}$ for $\{x_i\}$ . Apply the map \eqref{eqn:MVC} first to $\{x_i\}_{i=1}^N$ and $\{x_i\}_{i=N+1}^{2N}$ separately and obtain $M_1 = \text{MVC}\big(\{x_i\}_{i=1}^N, \{w_i\}_{i=1}^N\big)$ and $M_2 = \text{MVC}\big(\{x_i\}_{i=N+1}^{2N}, \{w_i\}_{i=N+1}^{2N}\big)$. Then apply the map \eqref{eqn:MVC} to $M_1$ and $M_2$ again and obtain $M = \text{MVC}\big(\{M_1, M_2\}, \{h_1, h_2\}\big)$. We will show that there exists $\{\tilde{w}_i\}_{i=1}^{2n}$ such that $M = \text{MVC}\big(\{x_i\}_{i=1}^{2N}, \{\tilde{w}_i\}_{i=1}^{2N}\big)$. Observe that 
\begin{align*}
    M & = \text{MVC}\big(\{M_1, M_2\}, \{h_1, h_2\}\big)\\
      & = \textbf{Exp}_p\Big(h_1\textbf{Log}_pM_1 + h_2\textbf{Log}_pM_2\Big) \\
      & = \textbf{Exp}_p\Bigg(h_1\textbf{Log}_p\Big(\textbf{Exp}_p\Big(\sum_{i=1}^N w_i\textbf{Log}_px_i\Big)\Big) + \\
        & \qquad  h_2\textbf{Log}_p\Big(\textbf{Exp}_p\Big(\sum_{i=N+1}^{2N} w_i\textbf{Log}_px_i\Big)\Big)\Bigg) \\
      & = \textbf{Exp}_p\Bigg( \sum_{i=1}^N h_1w_i \textbf{Log}_px_i + \sum_{i=N+1}^{2N} h_2w_i \textbf{Log}_px_i\Bigg)
\end{align*}
Hence, $\tilde{w}_i = h_1w_i$ for $i=1,\ldots,N$ and $\tilde{w}_i = h_2w_i$ for $i = N+1,\ldots,2N$ and the two layers collapse into a single layer.
\end{proof}
If we consider different normal charts for $\{x_i\}_{i=1}^N$ and $\{x_i\}_{i=N+1}^{2N}$, i.e.\ $U_{\bar{x}_1}$ for $\{x_i\}_{i=1}^N$ and $U_{\bar{x}_2}$ for $\{x_i\}_{i=N+1}^{2N}$, then the cascaded two layer structure will not collapse. However, to avoid any possibility of a collapse, e.g.\ in the case that $d(\bar{x}_1,\bar{x}_2) \approx 0$, we recommend the inclusion of a non-linear activation function between the layers. The choice of activation functions for manifold-valued input are however limited. As the most widely used activation function in CNN is ReLU, we propose to use the tangent ReLU (tReLU) \cite{chakraborty2019deep} as the activation function for the MVC-Net.

\subsection{Manifold-valued Fully-connected (MVFC) Layers for MVC-Net}

The outputs of the last MVC layer/tReLU layer would be a set of points on the manifold $M$. Therefore the desired FC layer should take points on the manifold as inputs and output labels (hard assignment) or probability vectors (soft assignment). In this work, we adopt the FC layer used in \citet{chakraborty2019deep}, i.e.\ for $\{x_1,\ldots,x_n\} \subset M$, first transform $\{x_1,\ldots,x_n\}$ to $\{d_g(x_1,\bar{x}), \ldots, d_g(x_n, \bar{x})\}$ and then apply the usual (Euclidean) FC layers as in CNN. 

\subsection{Architecture of MVC-Net}

For classification problems, the architecture we use in this work is parallel to CNN, i.e.\ 
\begin{align*}
\textbf{MVC + tReLU} \to & \textbf{MVC + tReLU} \to \cdots \to & \textbf{MVFC}.
\end{align*}
The number and the size of the layers will be presented in section~\ref{sec:experiment} as it depends on the experiment settings. Besides the classical CNN, different deep network architectures for data in Euclidean space have been proposed to solve specific application problems and the convolutional layer serves as the basic component in most of them. In a similar manner, for manifold-valued data, based on the application problem, we envision an appropriate architecture with MVC layers as the building blocks.

\section{Experiments} \label{sec:experiment}
In this section we present several experiments demonstrating the performance of the MVC-net. The experiments involve the use of data from medical imaging as well as computer vision domains. In all the experiments, we present comparisons to the state-of-the-art.

\subsection{Parkinson's Disease Classification}

\begin{figure*}[!ht] \centering \scalebox{0.9}{ \begin{tabular}{|c|c|c|c|c|c|}
    \topline\rowcolor{tableheadcolor} & & & {\bf time (s)} & \multicolumn{2}{c}{{\bf Accuracy}} \\
    \arrayrulecolor{tableheadcolor}\hhline{---~~~}\arrayrulecolor{rulecolor}\hhline{~~~---}\rowcolor{tableheadcolor}
    
    \multirow{-2}{*}{\bf Model} & \multirow{-2}{*}{\bf Non-linearity}  & \multirow{-2}{*}{\bf \# params.}  & {\bf / sample} & Training Accuracy & Test Accuracy\\ 
    
    \midtopline 
    {\textbf{MVC-net}} & TReLU & $\sim \highest{14K}$ & $\sim 0.13$ & $\highest{0.991 \pm 0.01}$ & $\highest{0.973 \pm 0.07}$  \\ 
    DTI-ManifoldNet  & None & $\sim 30K$ & $\sim 0.3$ & $\highest{0.973 \pm 0.02}$ & $\highest{0.948 \pm 0.03}$  \\ 
    ResNet-34 & ReLU & $\sim 30M$ & $\sim\highest{0.008}$ & $\highest{0.984 \pm 0.04}$ & $0.713 \pm 0.02$  \\ 
    CapsuleNet & ReLU & $\sim 30M$ & $\sim 0.009$ & $0.63 \pm 0.02$ & $0.62 \pm 0.04$  \\ 
    \bottomline \end{tabular} } 
    \caption{Comparison results on Diffusion MRI classification.} 
    \label{results:diff} 
\end{figure*}

\begin{figure*}[t!] 
    \centering
    \begin{subfigure}[t]{0.2\textwidth}
        \centering
        \includegraphics[height=1in]{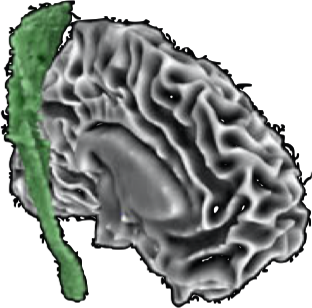}
        \caption{M1 Tract}
    \end{subfigure}%
    ~ 
    \begin{subfigure}[t]{0.2\textwidth}
        \centering
        \includegraphics[height=1in]{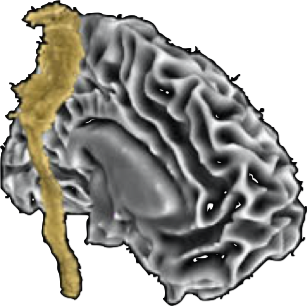}
        \caption{PMd Tract}
    \end{subfigure}
    ~
    \begin{subfigure}[t]{0.2\textwidth}
        \centering
        \includegraphics[height=1in]{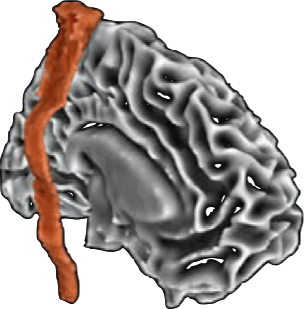}
        \caption{SMA Tract}
    \end{subfigure}
    \caption{SMATT \cite{Archer2017SMATTS} motor tract segmentation examples. }
    \label{M1_fig}
\end{figure*}

In this section, we apply the MVC-Net to a classification problem in the field of movement disorders, specifically, using diffusion magnetic resonance images (dMRIs) to classify Parkinson's disease (PD) patients from controls. 

\subsubsection*{\bf Diffusion MRI Data Acquisition and Pre-processing}

The dataset we use in this work consists of dMRIs acquired from 355 Parkinson's disease (PD) patients and 356 control (healthy) subjects. This data was acquired from a combination of three sources namely, (i) The University of Florida (UFL), (ii) The Parkinson’s Progression Markers Initiative (PPMI) database (\url{www.ppmi-info.org/data}) and (iii) The University of Michigan. The data acquired at UFL is publicly available for research use by request via the National Institute of Neurological Disorders (NINDS) Parkinson's Disease Biomarker Program (PDBP). This PDBP data contained images that were collected using a 3.0 T MR scanner and 32-channel quadrature volume head coil. The scanning parameters of the dMRIs acquisition sequence were as follows: gradient directions = $64$, b-values = $0/1000$ $s/\text{mm}^2$, resolution = $2$mm uniform voxel size.  The data from University of Michigan was obtained using a 3T Phillips MR scanner and the parameters were, gradient directions= $15$, b-values = $0/800$ $s/\text{mm}^2$, resolution = $1.75$mm uniform voxel size. Eddy current correction was applied to each data set by using standard motion correction techniques.

From each of these dMRIs, 12 regions of interest (ROIs) -- six on each hemisphere of the brain --  in the sensorimotor tract are segmented by registering to the sensorimotor area tract template (SMATT) \cite{Archer2017SMATTS}. These tracts are known to be affected by PD. Figure~\ref{M1_fig} depicts the M1, dorsal premotor cortex (PMd) and the supplementary motor area (SMA) tracts . In our experiments, we adopt the most widely used representation of dMRI in the clinic namely, diffusion tensor images. Diffusion tensors (DTs) are symmetric positive-definite matrices \cite{basser1994mr}. 

\subsubsection*{\bf Diffusion Tensor Representation}

The DTI representation of diffusion weighted images assumes a local Gaussian distribution of water diffusion within each voxel \cite{basser1994mr}. The covariance matrix of each local Gaussian represents the diffusion tensor, which is a symmetric positive definite (SPD) matrix. Thus we have a field $f : U \subset \mathbb{Z}^3 \to P_3$. We can equip the space $P_3$ with the $GL(3)$-invariant metric to make it a Riemannian homogeneous manifold. 

We estimate the diffusion tensor images from the segmented dMRIs of the sensorimotor tracts using the DiPy software \cite{DiPyPackage}. This data is fed directly into an MVC-net with five $\textbf{MVC} + \textbf{tReLU}$ layers. The output from the last of these layers forms the input to an $\textbf{MVFC}$ layer which maps this input into $\mathbb{R}^n$. Next, two standard fully connected layers are applied to this $\mathbb{R}^n$-valued input followed by a softmax function to output class probabilities. This architecture was found to give the best performance among similar architectures. 
\subsubsection*{\bf Classification Results}
We compared the performance of MVC-Net with 
several deep net architectures including the ManifoldNet \cite{chakraborty2019deep}, the ResNet-34 CNN architecture  and a CapsuleNet architecture with dynamic routing. To perform the comparison, we applied each of the aforementioned deep net architectures to the above described diffusion tensor image data sets. 

We train our MVC-net architecture for 200 epochs using cross-entropy loss and an Adam optimizer with learning rate set to $0.005$. We obtain a 10-fold cross validation accuracy of $97.8 \%$. For the ManifoldNet, we achieved  a 10-fold cross validation accuracy of $94.8 \%$. The ResNet-34 and CapsuleNet architectures are trained directly on the diffusion weighted images (without any diffusion tensor fitting to the dMRI data in the ROIs since they can not cope with symmetric positive definite matrix-valued images). With the ResNet-34 architecture we observe significant overfitting late in training and we utilize an early stoppage approach to report the best $10$-fold cross validation result, which still significantly under-performs the MVC-net and ManifoldNet (the only two approaches that respect the underlying geometry of $P_3$) respectively. Comprehensive results are reported in Table~\ref{results:diff}. 

As evident from the Table~\ref{results:diff}, MVC-net outperforms all other methods on both training and test accuracy while simultaneously keeping the lowest parameter count. The inference speed under-performs ResNet-34 and CapsuleNet, but these architectures utilize operations that have been optimized heavily for inference speed for years. Further, in terms of the possible application domain of automated Parkinson's diagnosis, the sub-second (less than a second) inference speeds we have achieved are more than sufficient in practice. 
\subsection{Anatomical Structure to Function Regression}

In this experiment we consider the problem of learning a function from a structural image of the human brain to a functional physiological measurement. Specifically, we consider the problem of mapping from Cauchy Deformation Tensor (CDT) images estimated relative to an atlas of diffusion MRI scans of the Substantia Nigra \cite{Banerjee2016}, a neuro-anatomical region known to be affected by movement disorders to MDS-UPDRS scores. The MDS-sponsored Revision of the Unified Parkinson's Disease Rating Scale (MDS-UPDRS) is a quantitative measure of PD severity assigned by a physician that combines various physical and psychological biomarkers associated with PD such as sleep quality, depression, and motor skills. The CDT of a diffusion MRI scan captures the deviation of a particular subject from a reference atlas (i.e. an "average" brain over the population), thus the CDT captures structural information about a particular brain.
\subsubsection*{\bf Data Acquisition and Pre-processing}

The data here consists of high angular resolution diffusion MRI (HARDI) \cite{TuchHARDI2002} images of 25 controls, 15 essential tremor (ET) patients and 26 PD patients acquired using the same parameters as the PDBP data in the previous experiment. For each patient we have corresponding MDS-UPDRS scores. We segment the Substantia Nigra (40 voxels large) from each of these images. Each image is pre-processed to estimate an Ensemble Average Propagator (EAP) at each voxel leading to an EAP field representation of the dMRI data. The EAP $P(\mathbf{x},\mathbf{r})$ is a probability distribution that describes the likelihood of water diffusing along a vector $\mathbf{r}$ \cite{johansen2013diffusion}.
To compute the CDT we follow a standard procedure which we outline here. First, we non-rigidly register \cite{cheng2009non} each of the EAP-field images to the Montreal Neurological Institute (MNI) reference atlas \cite{fonov2011unbiased}. Let $J$ be the Jacobian of the non-rigid registration, then the CDT at each voxel is given by $\sqrt{J J^T}$. This gives a $3 \times 3$ SPD matrix at each voxel, hence for each sample we have a $40 \times 3 \times 3$ dimensional tensor. Thus, to summarize, the independent variables are $40 \times 3 \times 3$ sized CDT fields describing structural properties of a particular human brain, and the dependent variables are the vector of MDS-UPDRS scores, quantifying functional severity of movement disorders.

We compare an MVC-net architecture operating on the space $P_3$ where the CDT descriptors live. The architecture for this problem consists of $3$ $\textbf{MVC} + \textbf{tReLU}$ layers followed by a $\textbf{MVFC}$ and two Euclidean fully connected layers plus a softmax layer. We compare the performance of the MVC-net to state-of-the-art methods for this task in \citet{chakraborty2019deep} and \citet{Banerjee2016}. The performance is quantified in terms of the $R^2$ statistic. Results are summarized in Table~\ref{regression}. 

\begin{figure}[!ht] \centering
    \begin{tabular}{|c|c|}
    \hhline{--}\rowcolor{tableheadcolor}
    
    {\bf Model} & {\bf $R^2$} \\ 
    \midtopline 
    
    \textbf{MVC-net} & $\highest{0.956}$   \\ 
    DTI-ManifoldNet \cite{chakraborty2019deep} & 0.930   \\ 
    NL Manifold Regression \cite{Banerjee2016} & 0.925 \\
    \bottomline
    \end{tabular} 
    \caption{\footnotesize Structure to Function Regression $R^2$-statistic.} 
    \label{regression} 
\end{figure}

As is evident from Table~\ref{regression}, MVC-net outperforms the competing methods on this particular task, although all methods perform well. Beyond this, MVC-net again achieves significant parameter efficiency, with $\sim 10 K$ parameters for this architecture. Future work will focus on evaluating MVC-net on this task for much larger datasets. 
\subsection{Video Classification}

We now outline an architecture for using MVC-net together with covariance blocks \cite{Yu2017SecondOrderCNN} to perform video classification. We present results of applying this MVC-net architecture to the Moving MNIST dataset, which is generated using the algorithm in \citet{Srivastava2015LSTM-Vid}. Each video consists of two MNIST digits moving across the frame. The velocity of both digits is fixed across all videos in a class, but the digits themselves  vary (in the range $0-9$). Different classes have different angles of motion, and the goal is to classify them based on this angle. 
\subsubsection*{\bf Architecture for Video Classification}
\begin{figure*}[!htb]
 \centering \scalebox{0.38}{
\hbox{ \includegraphics{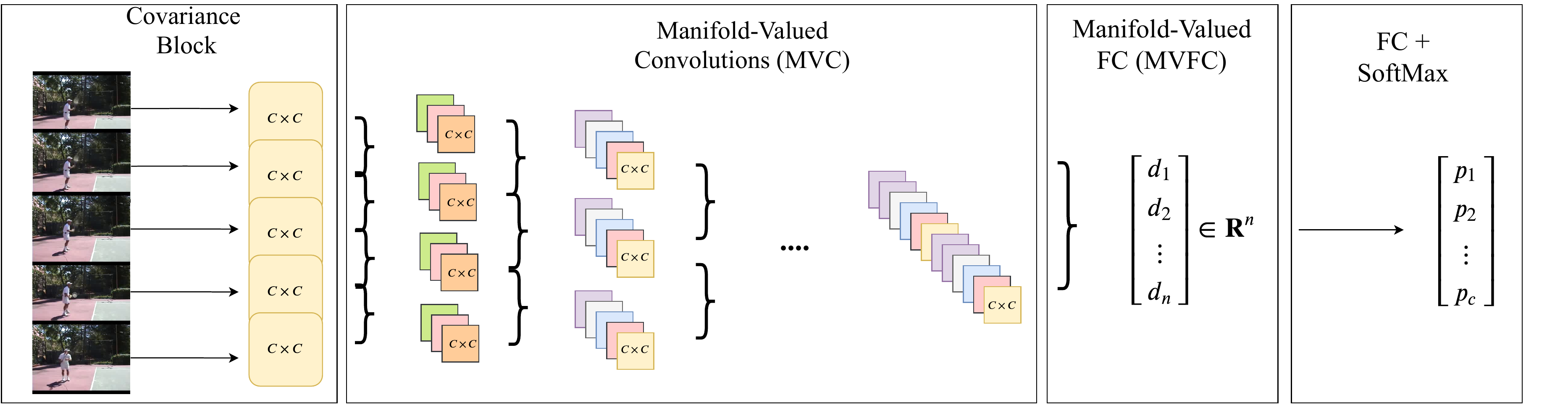}}}
\caption{MVC-net Video Classification Architecture}
\label{vid-class}
\end{figure*}

\begin{figure*}[!htb]  \centering \scalebox{0.9}{ \begin{tabular}{|c|c|c|c|c|c|}
    \topline\rowcolor{tableheadcolor} & & {\bf time (s)} & \multicolumn{3}{c}{{\bf
    orientation ($^\circ$)}} \\
    \arrayrulecolor{tableheadcolor}\hhline{---~~~}\arrayrulecolor{rulecolor}\hhline{~~~---}\rowcolor{tableheadcolor}
    \multirow{-2}{*}{\bf Mode} & \multirow{-2}{*}{\bf \# params.} & {\bf / epoch} &
    $30$-$60$ & $10$-$15$ & $10$-$15$-$20$ 
    \\ \midtopline MVC-Net & $13564$ & $\sim 4.1$ & $\highest{1.00 \pm 0.00}$ & $\highest{0.99 \pm 0.01}$  & $\highest{0.98
    \pm 0.01}$ \\ 
    Manifold DCNN & $1517$ & $\sim 4.3$ & $\highest{1.00 \pm 0.00}$ & $\highest{1.00 \pm 0.01}$  & $\highest{0.95
    \pm 0.01}$ \\
     SPD-TCN & $\highest{738}$ & $\sim 2.7$ & $\highest{1.00 \pm 0.00}$ & $\highest{0.99 \pm 0.01}$  & $\highest{0.97
    \pm 0.02}$ \\ 
    SPD-SRU & $1559$ & $\sim 6.2$ & $\highest{1.00\pm 0.00}$ & ${0.96 \pm
    0.02}$ & ${0.94 \pm 0.02}$ \\ 
    TT-GRU & $2240$ & $\sim\highest{2.0}$ & $\highest{1.00
    \pm 0.00}$ & $0.52 \pm 0.04$ & $0.47 \pm 0.03$ \\ 
    TT-LSTM & $2304$ &
    $\sim\highest{2.0}$ & $\highest{1.00 \pm 0.00}$ & $0.51 \pm 0.04$ & $0.37 \pm 0.02$ \\
    SRU & $159862$ & $\sim 3.5$ & $\highest{1.00 \pm 0.00}$ & $0.75 \pm 0.19$ & $0.73 \pm
    0.14$ \\ 
    LSTM & $252342$ & $\sim 4.5$ & $0.97 \pm 0.01$ & $0.71 \pm 0.07$ & $0.57 \pm
    0.13$ \\ 
    \bottomline \end{tabular} } \caption{\footnotesize Comparison results on
    Moving MNIST. All classification results are 10-fold cross validation test accuracy.} \label{video-res} 
    \end{figure*}
    
We now present an MVC-net architecture for video classification. Given an input video of dimensions $F \times 3 \times H \times W$, a covariance block \cite{Yu2017SecondOrderCNN} is applied in parallel to each frame to yield an $F \times (C+1) \times (C+1)$ tensor.
An illustration of the architecture is shown in Figure~\ref{vid-class}. We will now describe the components of this architecture.

For completeness, we will summarize the covariance block design from \citet{Yu2017SecondOrderCNN} below. The input to the covariance block is an image of size $3 \times H \times W$. We first apply a regular CNN without fully connected layers at the end to get a $C \times H' \times W'$ sized output. Now we interpret each channel as a feature vector and compute a $C \times C$ covariance matrix of the channel activations. Finally, to incorporate the first order statistics, we append the mean channel activation to both the last row and column of the covariance matrix to get a $(C + 1) \times (C+1)$ shaped output. 

As mentioned before, applying a covariance block at each frame of a video in parallel yields a $F \times (C+1) \times (C+1)$ shape tensor, where at each frame we have a $(C+1)\times (C+1)$ covariance matrix, which is an element in the space $P_{C+1}$. We now use a one-dimensional temporal MVC-net architecture to map the per-frame covariance descriptors to class outputs. This is no different than traditional temporal CNNs, i.e.\ at each layer, a moving window slides over the frames and computes a weighted combination. For our architecture, we use the manifold-valued convolution defined earlier on this sequence of frames each represented by a covariance matrix descriptor . Figure~\ref{vid-class} depicts a schematic of the MVC-net tailored for the video classification problem.

{\bf Experimental Results:}
For this experiment we use five $\textbf{MVC}+\textbf{tReLU}$ layers, followed by an \textbf{MVFC} layer and two Euclidean fully connected layers and a softmax. We use an Adam optimizer with learning rate set to $0.005$ and train for $300$ epochs using the cross-entropy loss. 10-fold cross validation results are summarized in Table~\ref{video-res}. As evident, the MVC-net either outperforms or is competitive with all competing methods in terms of test accuracy.

\section{Conclusion} \label{sec:conclusion}
In this paper, we presented a generalization of CNNs to manifold-valued images i.e., images whose value sets lie in Riemannian manifolds. Such data are commonly encountered in many applications including but not limited to medical imaging and computer vision. We defined the the analog of the traditional convolution operation to manifold-valued images and proved that it is equivariant to the isometry group actions admitted by the manifold. Equivariance is a fundamental design principle in traditional CNNs that affords weight sharing in the deep neural networks. Further, we also proved that a multi-layer MVC-Net requires the use of nonlinear activation functions and proposed a tangent-ReLU (tReLU) to this end. The final layer of the MVC-net is the manifold-valued fully connected layer whose construction is adopted from \citet{chakraborty2019deep}. Finally, we presented several experiments demonstrating the performance of the MVC-Net on classification problems drawn from medical imaging and computer vision. Comparisons to state-of-the art was presented demonstrating comparable to superior performance of the MVC-Net in terms of classification accuracy, parameter and time/epoch efficiency of the MVC-Net.  

\bibliography{references}

\begin{thebibliography}{27}
\providecommand{\natexlab}[1]{#1}
\providecommand{\url}[1]{\texttt{#1}}
\expandafter\ifx\csname urlstyle\endcsname\relax
  \providecommand{\doi}[1]{doi: #1}\else
  \providecommand{\doi}{doi: \begingroup \urlstyle{rm}\Url}\fi

\bibitem[Afsari(2011)]{Afsari2011}
Afsari, B.
\newblock {Riemannian ${L}^p$ center of mass: Existence, uniqueness, and
  convexity}.
\newblock \emph{Proceedings of the American Mathematical Society}, 139\penalty0
  (02):\penalty0 655--655, 2011.
\newblock ISSN 0002-9939.
\newblock \doi{10.1090/S0002-9939-2010-10541-5}.

\bibitem[Archer et~al.(2017)Archer, Vaillancourt, and
  Coombes]{Archer2017SMATTS}
Archer, D., Vaillancourt, D., and Coombes, S.
\newblock A template and probabilistic atlas of the human sensorimotor tracts
  using diffusion mri.
\newblock \emph{Cerebral Cortex}, 28:\penalty0 1--15, 03 2017.
\newblock \doi{10.1093/cercor/bhx066}.

\bibitem[Banerjee et~al.(2016)Banerjee, Chakraborty, Ofori, Okun, Viallancourt,
  and Vemuri]{Banerjee2016}
Banerjee, M., Chakraborty, R., Ofori, E., Okun, M.~S., Viallancourt, D.~E., and
  Vemuri, B.~C.
\newblock {A nonlinear regression technique for manifold valued data with
  applications to medical image analysis}.
\newblock In \emph{Proceedings of the IEEE Conference on Computer Vision and
  Pattern Recognition}, pp.\  4424--4432, 2016.

\bibitem[Banerjee et~al.(2019)Banerjee, Chakraborty, Archer, Vaillancourt, and
  Vemuri]{banerjee2019dmr}
Banerjee, M., Chakraborty, R., Archer, D., Vaillancourt, D., and Vemuri, B.~C.
\newblock Dmr-cnn: A cnn tailored for dmr scans with applications to pd
  classification.
\newblock In \emph{2019 IEEE 16th International Symposium on Biomedical Imaging
  (ISBI 2019)}, pp.\  388--391. IEEE, 2019.

\bibitem[Basser et~al.(1994)Basser, Mattiello, and LeBihan]{basser1994mr}
Basser, P.~J., Mattiello, J., and LeBihan, D.
\newblock Mr diffusion tensor spectroscopy and imaging.
\newblock \emph{Biophysical journal}, 66\penalty0 (1):\penalty0 259--267, 1994.

\bibitem[Bronstein et~al.(2017)Bronstein, Bruna, LeCun, Szlam, and
  Vandergheynst]{bronstein2017geometric}
Bronstein, M.~M., Bruna, J., LeCun, Y., Szlam, A., and Vandergheynst, P.
\newblock {Geometric deep learning: going beyond euclidean data}.
\newblock \emph{IEEE Signal Processing Magazine}, 34\penalty0 (4):\penalty0
  18--42, 2017.

\bibitem[Brooks et~al.(2019)Brooks, Schwander, Barbaresco, Schneider, and
  Cord]{brooks2019riemannian}
Brooks, D., Schwander, O., Barbaresco, F., Schneider, J.-Y., and Cord, M.
\newblock Riemannian batch normalization for spd neural networks.
\newblock \emph{arXiv preprint arXiv:1909.02414}, 2019.

\bibitem[Chakraborty et~al.(2019)Chakraborty, Bouza, Manton, and
  Vemuri]{chakraborty2019deep}
Chakraborty, R., Bouza, J., Manton, J., and Vemuri, B.~C.
\newblock A deep neural network for manifold-valued data with applications to
  neuroimaging.
\newblock In \emph{International Conference on Information Processing in
  Medical Imaging}, pp.\  112--124. Springer, 2019.

\bibitem[Cheng et~al.(2009)Cheng, Vemuri, Carney, and Mareci]{cheng2009non}
Cheng, G., Vemuri, B.~C., Carney, P.~R., and Mareci, T.~H.
\newblock Non-rigid registration of high angular resolution diffusion images
  represented by gaussian mixture fields.
\newblock In \emph{International Conference on Medical Image Computing and
  Computer-Assisted Intervention}, pp.\  190--197. Springer, 2009.

\bibitem[Cohen et~al.(2018)Cohen, Geiger, K{\"{o}}hler, and
  Welling]{cohen2018spherical}
Cohen, T.~S., Geiger, M., K{\"{o}}hler, J., and Welling, M.
\newblock {Spherical CNNs}.
\newblock \emph{arXiv preprint arXiv:1801.10130}, 2018.

\bibitem[Esteves et~al.(2018)Esteves, Allen-Blanchette, Makadia, and
  Daniilidis]{esteves2018learning}
Esteves, C., Allen-Blanchette, C., Makadia, A., and Daniilidis, K.
\newblock Learning so (3) equivariant representations with spherical cnns.
\newblock In \emph{Proceedings of the European Conference on Computer Vision
  (ECCV)}, pp.\  52--68, 2018.

\bibitem[Fonov et~al.(2011)Fonov, Evans, Botteron, Almli, McKinstry, Collins,
  and Group]{fonov2011unbiased}
Fonov, V., Evans, A.~C., Botteron, K., Almli, C.~R., McKinstry, R.~C., Collins,
  D.~L., and Group, B. D.~C.
\newblock Unbiased average age-appropriate atlases for pediatric studies.
\newblock \emph{Neuroimage}, 54\penalty0 (1):\penalty0 313--327, 2011.

\bibitem[Garyfallidis et~al.(2014)Garyfallidis, Brett, Amirbekian, Rokem, Van
  Der~Walt, Descoteaux, and Nimmo-Smith]{DiPyPackage}
Garyfallidis, E., Brett, M., Amirbekian, B., Rokem, A., Van Der~Walt, S.,
  Descoteaux, M., and Nimmo-Smith, I.
\newblock Dipy, a library for the analysis of diffusion mri data.
\newblock \emph{Frontiers in Neuroinformatics}, 8:\penalty0 8, 2014.
\newblock ISSN 1662-5196.
\newblock \doi{10.3389/fninf.2014.00008}.

\bibitem[Groisser(2004)]{Groisser2004}
Groisser, D.
\newblock {Newton's method, zeroes of vector fields, and the Riemannian center
  of mass}.
\newblock \emph{Advances in Applied Mathematics}, 33\penalty0 (1):\penalty0
  95--135, 2004.
\newblock ISSN 01968858.
\newblock \doi{10.1016/j.aam.2003.08.003}.

\bibitem[Huang \& {Van Gool}(2017)Huang and {Van Gool}]{huang2017riemannian}
Huang, Z. and {Van Gool}, L.~J.
\newblock {A Riemannian Network for SPD Matrix Learning.}
\newblock In \emph{AAAI}, volume~1, pp.\ ~3, 2017.

\bibitem[Huang et~al.(2018)Huang, Wu, and Van~Gool]{huang2018building}
Huang, Z., Wu, J., and Van~Gool, L.
\newblock Building deep networks on grassmann manifolds.
\newblock In \emph{Thirty-Second AAAI Conference on Artificial Intelligence},
  2018.

\bibitem[Johansen-Berg \& Behrens(2013)Johansen-Berg and
  Behrens]{johansen2013diffusion}
Johansen-Berg, H. and Behrens, T.~E.
\newblock \emph{Diffusion MRI: from quantitative measurement to in vivo
  neuroanatomy}.
\newblock Academic Press, 2013.

\bibitem[Kondor \& Trivedi(2018)Kondor and Trivedi]{kondor2018generalization}
Kondor, R. and Trivedi, S.
\newblock {On the generalization of equivariance and convolution in neural
  networks to the action of compact groups}.
\newblock \emph{arXiv preprint arXiv:1802.03690}, 2018.

\bibitem[Kondor et~al.(2018)Kondor, Lin, and Trivedi]{kondor2018clebsch}
Kondor, R., Lin, Z., and Trivedi, S.
\newblock Clebsch--gordan nets: a fully fourier space spherical convolutional
  neural network.
\newblock In \emph{Advances in Neural Information Processing Systems}, pp.\
  10117--10126, 2018.

\bibitem[Lee(2006)]{lee2006riemannian}
Lee, J.~M.
\newblock \emph{Riemannian manifolds: an introduction to curvature}, volume
  176.
\newblock Springer Science \& Business Media, 2006.

\bibitem[Mallat(2016)]{Mallat2016}
Mallat, S.
\newblock {Understanding Deep Convolutional Networks}.
\newblock \emph{Philosophical Transactions A}, 374:\penalty0 20150203, 2016.
\newblock ISSN 1364503X.
\newblock \doi{10.1098/rsta.2015.0203}.

\bibitem[Masci et~al.(2015)Masci, Boscaini, Bronstein, and
  Vandergheynst]{masci2015geodesic}
Masci, J., Boscaini, D., Bronstein, M., and Vandergheynst, P.
\newblock {Geodesic convolutional neural networks on riemannian manifolds}.
\newblock In \emph{Proc. of the {IEEE} Intl. Conf. on Computer Vision
  workshops}, pp.\  37--45, 2015.

\bibitem[{Maurice Fr{\'{e}}chet}(1948)]{Frechet1948a}
{Maurice Fr{\'{e}}chet}.
\newblock {Les {\'{e}}l{\'{e}}ments al{\'{e}}atoires de nature quelconque dans
  un espace distanci{\'{e}}}.
\newblock \emph{Annales de l'I. H. P.,}, 10\penalty0 (4):\penalty0 215--310,
  1948.

\bibitem[Poulenard \& Ovsjanikov(2018)Poulenard and
  Ovsjanikov]{poulenard2018multi}
Poulenard, A. and Ovsjanikov, M.
\newblock Multi-directional geodesic neural networks via equivariant
  convolution.
\newblock In \emph{SIGGRAPH Asia 2018 Technical Papers}, pp.\  236. ACM, 2018.

\bibitem[Srivastava et~al.(2015)Srivastava, Mansimov, and
  Salakhutdinov]{Srivastava2015LSTM-Vid}
Srivastava, N., Mansimov, E., and Salakhutdinov, R.
\newblock Unsupervised learning of video representations using lstms.
\newblock In \emph{Proceedings of the 32Nd International Conference on
  International Conference on Machine Learning - Volume 37}, ICML'15, pp.\
  843--852. JMLR.org, 2015.

\bibitem[Tuch et~al.(2002)Tuch, Reese, Wiegell, Makris, Belliveau, and
  Wedeen]{TuchHARDI2002}
Tuch, D.~S., Reese, T.~G., Wiegell, M.~R., Makris, N., Belliveau, J.~W., and
  Wedeen, V.~J.
\newblock High angular resolution diffusion imaging reveals intravoxel white
  matter fiber heterogeneity.
\newblock \emph{Magnetic Resonance in Medicine}, 48\penalty0 (4):\penalty0
  577--582, 2002.
\newblock \doi{10.1002/mrm.10268}.

\bibitem[{Yu} \& {Salzmann}(2017){Yu} and {Salzmann}]{Yu2017SecondOrderCNN}
{Yu}, K. and {Salzmann}, M.
\newblock {Second-order Convolutional Neural Networks}.
\newblock \emph{ArXiv e-prints}, March 2017.

\end{thebibliography}
\bibliographystyle{icml2020}

\end{document}